\documentclass[11pt]{article}
\RequirePackage[colorlinks,citecolor=blue,urlcolor=blue,linkcolor=blue]{hyperref}
\usepackage{url}

\DeclareSymbolFont{slenderlargesymbols}{OMX}{ccex}{m}{n}
\DeclareMathSymbol{\prod}{\mathop}{slenderlargesymbols}{"51}
\usepackage{latexsym,amssymb,amsmath,amsthm,graphics,graphicx,float,psfrag, epsfig, color, authblk, enumerate}

\usepackage{pgfplots}
\usepackage{url,amssymb,amsmath,amsthm,amscd,paralist,bbm,wrapfig}
\usepackage{bm}
\usepackage{algorithm}
\usepackage{algorithmic} 

\usepackage{mathtools}

\usepackage[T1]{fontenc}
\usepackage[numbers]{natbib}
\tikzstyle{vertex}=[circle,black, fill=black, draw, inner sep=0pt, minimum size=6pt]

\usepackage{tikz}
\usetikzlibrary{decorations.pathreplacing}
\usepackage{xcolor}
\usepackage{mathtools}
\definecolor{cof}{RGB}{219,144,71}
\definecolor{pur}{RGB}{186,146,162}
\definecolor{greeo}{RGB}{91,173,69}
\definecolor{greet}{RGB}{52,111,72}
\pgfplotsset{compat=1.14}

\newtheorem{innercustomgeneric}{\customgenericname}
\providecommand{\customgenericname}{}
\newcommand{\newcustomtheorem}[2]{%
  \newenvironment{#1}[1]
  {%
   \renewcommand\customgenericname{#2}%
   \renewcommand\theinnercustomgeneric{##1}%
   \innercustomgeneric
  }
  {\endinnercustomgeneric}
}

\newcustomtheorem{customasmp}{Assumptions}

\newcommand{\R}{\mathbb{R}}
\newcommand{\N}{\mathcal{N}}

\newcommand{\Lcal}{\mathcal{L}}
\newcommand{\RC}{\mathsf{R}\mathsf{C}}
\newcommand{\Pro}{\mathbb{P}}
\newcommand{\one}{\operatorname{\mathbbm{1}}}
\newcommand{\Scal}{\mathcal{S}}
\newcommand{\T}{\mathrm{T}}
\renewcommand{\top}{\mathrm{T}}
\newcommand{\sgn}{\text{sgn}}
\newcommand{\diag}{\text{diag}}
\newcommand{\al}{\alpha}
\renewcommand{\epsilon}{\varepsilon}
\newcommand{\E}{\mathbb{E}}

\newcommand{\Bcal}{\mathcal{B}}
\newcommand{\dist}{\text{dist}}

\newtheorem{theorem}{Theorem}
\newtheorem{lemma}{Lemma}

\newtheorem{proposition}{Proposition}

\newtheorem{definition}{Definition}

\begin{document}
\title{Optimal Sample Complexity of Subgradient Descent for Amplitude Flow via Non-Lipschitz Matrix Concentration}
\author{Paul Hand\thanks{Department of Mathematics and College of Computer and Information Science, Northeastern University, Boston, MA}, Oscar Leong\thanks{Department of Computational and Applied Mathematics, Rice University, Houston, TX}, and Vladislav Voroninski\thanks{Helm.ai, Menlo Park, CA}}
\maketitle

\begin{abstract}

    We consider the problem of recovering a real-valued $n$-dimensional signal from $m$ phaseless, linear measurements and analyze the amplitude-based non-smooth least squares objective. We establish local convergence of subgradient descent with optimal sample complexity based on the uniform concentration of a random, discontinuous matrix-valued operator arising from the objective's gradient dynamics. While common techniques to establish uniform concentration of random functions exploit Lipschitz continuity, we prove that the discontinuous matrix-valued operator satisfies a uniform matrix concentration inequality when the measurement vectors are Gaussian as soon as $m = \Omega(n)$ with high probability. We then show that satisfaction of this inequality is sufficient for subgradient descent with proper initialization to converge linearly to the true solution up to the global sign ambiguity. As a consequence, this guarantees local convergence for Gaussian measurements at optimal sample complexity. The concentration methods in the present work have previously been used to establish recovery guarantees for a variety of inverse problems under generative neural network priors. This paper demonstrates the applicability of these techniques to more traditional inverse problems and serves as a pedagogical introduction to those results.
\end{abstract}
\section{Introduction}

Consider the problem of recovering a signal $x_* \in \R^n$ from $m$ phaseless measurements of the form \begin{align*}
    y := |Ax_*| + \eta
\end{align*} where $A \in \R^{m \times n}$ is a measurement matrix, $|\cdot|$ acts entrywise, and $\eta \in \R^m$ denotes noise. This problem is known as phase retrieval as, in practice, the phase of the signal is lost in the forward measurement process due to the underlying physics of the measurement system. We consider the case when the entries of $A$ are i.i.d. Gaussian, which we will refer to as the generic measurement regime. In this work, we aim to recover $x_*$ by solving the following non-smooth least squares problem \begin{align}
    \min_{x \in \R^n} f(x) := \frac{1}{2}\big\||Ax| - y\big\|^2. \label{amp_flow_obj}
\end{align} This objective function is known as Amplitude Flow.  For generic measurements, previous works have shown that with proper initialization, gradient descent both with \cite{Wang2018} and without \cite{AmpFlow} truncated gradients can recover the signal with the optimal sample complexity of $m = \Omega(n)$.

Existing proof techniques of convergence guarantees for (sub)gradient descent of \eqref{amp_flow_obj} follow a two-step process: 1) establish that spectral initialization or some variant thereof guarantees an initializer with relative error bounded by a small absolute constant and then 2) show that the objective satisfies a property akin to convexity near the minimizer to guarantee convergence. This latter property is called the \textit{local regularity condition} $\RC(\mu,\lambda,\epsilon)$\footnote{A function $\Lcal$ satisfies $\RC(\mu,\lambda,\epsilon)$ at a stationary point $y$ if for all $x \in \R^n$ such that $\|x-y\|\leqslant\epsilon\|y\|$, $\langle \nabla \Lcal(x),x-y\rangle \geqslant \frac{\mu}{2}\|x-y\|^2 + \frac{\lambda}{2}\|\nabla \Lcal(x)\|^2.$}. Showing that this condition holds is crucial in establishing local convergence for Amplitude Flow \cite{AmpFlow, Wang2018} and its variants \cite{SmoothAmpFlow}.

This proof technique is not unique to Amplitude Flow as it was initially introduced to guarantee convergence for the intensity-based formulation Wirtinger Flow, which aims to solve  \begin{align}
    \min_{x \in \R^n} \frac{1}{2}\big\||Ax|^2 - y^2\big\|^2. \label{wirtflowfunc}
\end{align} In the original work \cite{CandesLiSoltan2015}, the aforementioned two-step procedure established exact recovery with sample complexity $m = \Omega(n\log n)$. A follow-up variant using truncation \cite{ChenCandes2015} improved the sample complexity to $m = \Omega(n)$ and also employed the $\RC(\mu,\lambda,\epsilon)$ to show convergence post-initialization. Recently, \cite{YonelYazici2020} established the sufficiency of a deterministic condition for local convergence when solving \eqref{wirtflowfunc} in the lifted domain and a relationship between the condition's accuracy and convergence rate of gradient descent was shown. This deterministic condition is a uniform matrix concentration inequality that is proven to hold for generic measurements when $m = \Omega(n \log n)$. 

Other works include the global landscape analysis in \cite{Wright2016} which showed that \eqref{wirtflowfunc} exhibits benign geometry given a sufficient number of measurements $m = \Omega(n \log^3 n)$ by carefully analyzing the gradient and Hessian in partitioned regions of space. In \cite{DuchiRuan2018}, the authors considered the robust $\ell_1$ loss with intensity-based measurements and established local convergence of a prox-linear algorithm using composite optimization theory. Convex approaches based on lifting \cite{CSV2013, CandLi2014, DemanetHand14, Hand17} utilize dual certificates to assert correctness of the minimizers of semidefinite programs. Linear programming approaches have also been studied, whose proof techniques range from using tools in statistical learning theory \cite{BahmRomConvPR} and geometric probability theory \cite{Phasemax}, along with elementary approaches using standard concentration estimates of the singular values of random matrices \cite{HV16Elem}. For a more comprehensive overview of prior work for phase retrieval, we refer the reader to \cite{Numerics_of_PR}.

In this paper, we present a proof technique for solving \eqref{amp_flow_obj} with $m = \Omega(n)$ based on uniform concentration of random matrix-valued functions that are discontinuous in space. Consider a subgradient descent algorithm with iterates $\{x_t\}_{t \geqslant 0}$ of the form $x_{t+1} = x_t - \al v_{x_t,x_*}$ where $\al > 0$, $v_{x_t,x_*} \in \partial f(x_t)$, and $\partial f(x)$ is the Clarke subdifferential at $x$ (defined in Section \ref{proofs_section}). Let $\dist(x,x_*) := \min(\|x-x_*\|,\|x+x_*\|).$ We first state our main local convergence result in the Gaussian measurement regime. 

\begin{theorem} \label{main_prob_conv_result} There exists positive absolute constants $C$, $c_1$, $c_2$, $\rho_1,$ and $\rho_2$ such that the following holds. Suppose $A \in \R^{m \times n}$ has i.i.d. $\mathcal{N}(0,1/m)$ entries and the noise is bounded $\|\eta\|\leqslant \rho_1\|x_*\|$. Assume the initial iterate $x_0$ satisfies $\dist(x_0,x_*) \leqslant \rho_2\|x_*\|$ and the step size satisfies $0 < \al \leqslant 1$. If $m \geqslant C n$, then with probability at least $1 - 3\exp(-c_1 m)- m\exp(-c_2 n)$, we have that for all $t \geqslant 1$, $$\dist(x_{t},x_*) \leqslant \left(1 - \frac{\al}{2}\right)^{t}\dist(x_0,x_*) + 4\|\eta\|.$$
\end{theorem}

\noindent This result asserts local convergence up to the noise level with optimal sample complexity. In the theorem, note that we require an initializer with relative error less than a sufficiently small constant. There are several schemes to achieve this with $m = \Omega(n)$ Gaussian measurements, even in the presence of noise \cite{ChenCandes2015, AmpFlow, Wang2018, DuchiRuan2018}. While convergence of subgradient descent without truncation for the Amplitude Flow objective is known \cite{AmpFlow}, the method of proof we present here is novel. In particular, we show that Theorem \ref{main_prob_conv_result} is a consequence of the following two results: 1) a uniform matrix concentration inequality is sufficient to guarantee local convergence with proper initialization and 2) Gaussian matrices satisfy this inequality with high probability when $m = \Omega(n)$.

We now detail the high level intuition behind the proof ideas and techniques. Let $\sgn(z) := z/|z|$ for $z \neq 0$ and $\sgn(0)=0$ act entrywise. For ease of exposition, suppose there is no noise $\eta = 0$. The discontinuous, spatially-varying measurement operator $A_x := \diag(\sgn(Ax))A$ plays a critical role in analyzing subgradient descent as this operator governs the gradient dynamics of $f$. Specifically, the gradient almost everywhere is given by $\nabla f(x) = A_x^{\T}(A_x x - A_{x_*} x_*).$ As will be shown in the next section, the gradient in expectation obeys a property equivalent to the $\RC(\mu,\lambda,\epsilon)$ in neighborhoods of the global minimizers $\pm x_*$. Hence if we establish concentration of the quantity $A_x^{\T}A_y$ to its expectation $\E[A_x^{\T}A_y]$ uniformly in $x,y$, then this property will also be satisfied by the gradient. This will be shown to guarantee local convergence up to the global sign ambiguity with high probability.


The uniform concentration result we establish is the following: when $A$ has i.i.d. $\mathcal{N}(0,1/m)$ entries, then for any parameter $0 < \epsilon < 1$, when $m = \Omega(n)$ we have that with high probability \begin{align}\|A_x^{\T}A_y - \Phi_{x,y}\| \leqslant \epsilon\ \forall x,y \in \R^n \label{AxTAyapproxPhi}\end{align} where $\Phi_{x,y} := \E[A_x^{\T}A_y]$ has an analytic expression. As this result holds uniformly in $x,y$, we have that for any $x,x_* \in \R^n$, the gradient $\nabla f(x) \approx \Phi_{x,x}x - \Phi_{x,x_*}x_*$. The difficulty of establishing \eqref{AxTAyapproxPhi} uniformly in $x,y$ is due to the fact that $A_x^{\T}A_y$ is a non-Lipschitz matrix-valued operator. Standard approaches to control these types of quantities exploit Lipschitz continuity by first 1) establishing concentration for fixed $x,y$, then 2) establishing concentration over all points in a net of the sphere by using a union bound, and finally 3) appealing to Lipschitz continuity to get concentration uniformly in $x,y$. However, in this case, 3) is not possible as $A_x^{\T}A_y$ is discontinuous with respect to $x,y$.

Fortunately, this issue can be solved by concentrating Lipschitz continuous approximations of $A_x^{\T}A_y$ with respect to $x,y$. In particular, one can create continuous matrix-valued functions that are upper and lower bounds of $A_x^{\T}A_y$ with respect to the semidefinite ordering. Then, concentration of these continuous approximations can be established by appealing to the standard arguments outlined above. This, in turn, will be shown to establish concentration of the non-Lipschitz quantity of interest by a squeezing argument. Moreover, using novel tools developed in \cite{Daskalaskisetal2020}, a more efficient set of coverings of the sphere can be exploited to achieve sample complexity linear in the ambient dimension $n$. Intuitively, this is achieved by constructing a net of the sphere that does not penalize all directions equally, but instead exploits directions for which the function of interest does not deviate much and penalizes those for which the function exhibits larger change. This intuition is made precise in the proof of Proposition \ref{MDC_prop}.

\paragraph{Discussion.} While the convergence results presented here have been shown in \cite{AmpFlow}, the contribution of this work lies in the novelty of the analysis. In particular, this work is an illustrative example of using the concentration of non-Lipschitz functions to establish favorable properties of first-order algorithms to solve inverse problems. The concentration methods of this paper have been used to establish recovery in compressive sensing \cite{HV2017, Huangetal2018}, phase retrieval \cite{HLV18, HLV20}, and other problems \cite{Heckeletal2018, Qiuetal19, JoshiHand19, CHV20} under image priors given by generative neural networks. For example, a similar uniform matrix concentration inequality to the one used in this paper was introduced by the present authors in \cite{HLV18, HLV20} to establish recovery in compressive phase retrieval under a generative prior with information-theoretically optimal sample complexity.  This paper demonstrates the applicability of these techniques to more traditional inverse problems and serves as a pedagogical introduction to those results.

\section{Proof Technique}

We now establish the sufficiency of a deterministic condition for local convergence in the form of a uniform matrix concentration inequality and show that Gaussian matrices satisfy this condition with high probability when $m = \Omega(n)$. A similar condition, known as the Weight Distribution Condition, was first introduced in \cite{HV2017} in the context of compressive sensing under generative neural network priors. The matrix concentration inequality is stated as follows. 
 \begin{definition} Fix $0 < \epsilon < 1$. We say that $A \in \R^{m \times n}$ satisfies the \textbf{Measurement Distribution Condition (MDC)} with constant $\epsilon$ if \begin{align*}
    \|A_x^{\T}A_y - \Phi_{x,y}\|\leqslant \epsilon\ \forall\ x,y \in \R^n
\end{align*} where \begin{align}
    \Phi_{x,y}:=  \begin{cases} \frac{\pi - 2\theta_{x,y}}{\pi}I_n + \frac{2\sin \theta_{x,y}}{\pi} M_{\hat{x} \leftrightarrow \hat{y}}  & \text{ if } x\neq 0, y \neq 0, \\
    0 & \text{ otherwise.}
\end{cases} \label{Phi_def} \end{align} Here $\theta_{x,y} := \angle(x,y)$, $\Hat{x}: = x/\|x\|$, $\hat{y} := y/\|y\|$,  $I_n$ is the $n \times n$ identity matrix, and $M_{\hat{x}\leftrightarrow \Hat{y}}$\footnote{A formula for this matrix is as follows: consider a rotation matrix $R$ that sends $\hat{x} \mapsto e_1$ and $\hat{y} \mapsto \cos\theta_{x,y} e_1 + \sin\theta_{x,y} e_2$ where $\theta_{x,y} = \angle(x,y)$. Then $M_{\hat{x} \leftrightarrow \hat{y}} = R^\top \left[\begin{array}{ccc}
    \cos\theta_{x,y} & \sin\theta_{x,y} & 0  \\
     \sin \theta_{x,y} & -\cos \theta_{x,y} & 0 \\
     0 & 0 & 0_{n-2}
\end{array}\right]R$ where $0_{n-2}$ is the $n-2 \times n-2$ matrix of zeros. Note that if $\theta_{x,y} = 0$ or $\pi$, $M_{\hat{x} \leftrightarrow \hat{y}} = \hat{x}\hat{x}^\top$ or $-\hat{x}\hat{x}^\top$, respectively.} is the matrix that sends $\Hat{x} \mapsto \Hat{y}$, $\Hat{y} \mapsto \Hat{x}$, and $z \mapsto 0$ for any $z \in \text{span}(\{x,y\})^{\perp}.$
\end{definition} 

\noindent Note that for points $x,y$ with small angle, this condition requires $A_x^{\T}A_y$ to act like an isometry. In the extreme case when $x = y$, $\Phi_{x,y}$ is the identity. An elementary calculation gives $\E[A_x^{\T}A_y] = \Phi_{x,y}$ for $x,y \neq 0$ and $A_{ij} \sim \N(0,1/m)$. 

The first result is that the MDC is sufficient to guarantee the following: a subgradient descent algorithm with proper initialization will converge to the true solution up to the global sign ambiguity. To establish this, we first show that the MDC is sufficient for the objective to satisfy the following regularity condition which states that, within a neighborhood of the true solution, all subgradients point towards the true solution. This result is proven in Section \ref{proofs_section_conv}. \begin{lemma} \label{convexity_Lemma} Fix $0 < \epsilon \leqslant 0.001$. Suppose $A \in \R^{m \times n}$ satisfies the MDC with constant $\epsilon$. Then for all $x \in \R^n$ such that $\dist(x,x_*) \leqslant \epsilon\|x_*\|$ and any $v_{x,x_*} \in \partial f(x)$, we have that $\| v_{x,x_*} -  (x \pm x_*) \| \leqslant \frac{1}{2}\|x\pm x_*\| + 2\|\eta\|.$ Here $x \pm x_* := x-x_*$ if $\|x-x_*\| = \dist(x,x_*)$ and $x+x_*$ otherwise.
\end{lemma} \noindent Note that the conclusion of this lemma in the noiseless setting is in fact equivalent to the  $\RC(\mu,\lambda,\epsilon)$ condition\footnote{Indeed, note that if the conditions of Lemma \ref{convexity_Lemma} are satisfied and $\eta = 0$, then for all $x \in \R^n$ such that $\dist(x,x_*)\leqslant \epsilon\|x_*\|$ and any $v_{x,x_*} \in \partial f(x)$, 
    $\|v_{x,x_*}-(x\pm x_*)\| \leqslant \frac{1}{2}\|x\pm x_*\| \Longleftrightarrow 
    \langle v_{x,x_*},x\pm x_*\rangle  \geqslant \frac{3}{8}\|x\pm x_*\|^2 + \frac{1}{2}\|v_{x,x_*}\|^2$. Thus the MDC is sufficient to guarantee the $\RC(\mu,\lambda,\epsilon)$ holds with $\mu = 3/4$, $\lambda =1$, and our choice of $\epsilon$.}. We now show that satisfaction of the MDC implies local convergence.

\begin{theorem}[Deterministic Local Convergence Guarantee] \label{conv_to_min} 
Fix $0 < \epsilon \leqslant 0.001$. Suppose $A \in \R^{m \times n}$ satisfies the MDC with constant $\epsilon$, $\|\eta\| \leqslant \frac{\epsilon}{4}\|x_*\|$, and $0 < \al \leqslant 1$. If  $\dist(x_0,x_*)\leqslant\epsilon\|x_*\|$ then for all $t \geqslant 1$, $$\dist(x_{t},x_*) \leqslant \left(1 - \frac{\al}{2}\right)^{t}\dist(x_0,x_*) + 4\|\eta\|.$$
\end{theorem}
\begin{proof}[Proof of Theorem \ref{conv_to_min}] Let $\Bcal(x_*,r) := \{x \in \R^n : \|x-x_*\|\leqslant r\}$. 
Suppose $x_0 \in \mathcal{B}(x_*,\epsilon\|x_*\|)$ as the proof for the case $x_0 \in \mathcal{B}(-x_*,\epsilon\|x_*\|)$ is identical. For $t \geqslant 1$, observe that for any $v_{x_{t-1},x_*} \in \partial f(x_{t-1})$, we have \begin{align}
    \|x_{t} - x_*\| & = \|x_{t-1} - \al v_{x_{t-1},x_*} + \al(x_{t-1}-x_*) - \al(x_{t-1}-x_*) - x_*\| \nonumber\\
    & \leqslant \left(1-\al\right)\|x_{t-1}-x_*\| + \al\|v_{x_{t-1},x_*} - (x_{t-1}-x_*)\| \nonumber\\
    & \leqslant \left(1-\al\right)\|x_{t-1}-x_*\| + \frac{\al}{2}\|x_{t-1}-x_*\| + 2\al\|\eta\| \nonumber\\
    & = \left(1 - \frac{\al}{2}\right)\|x_{t-1}-x_*\| + 2\al\|\eta\| \label{iterate_bound}
\end{align} where in the third line we used Lemma \ref{convexity_Lemma}. We claim that the iterates must stay within a ball of the minimizer. Indeed, if $x_{t-1} \in \Bcal(x_*,\epsilon\|x_*\|)$, we have that by equation \eqref{iterate_bound} and our bound on the size of the noise $\|\eta\| \leqslant \frac{\epsilon}{4}\|x_*\|$ that \begin{align*}
    \|x_t-x_*\| & \leqslant \left(1 - \frac{\al}{2}\right)\|x_{t-1}-x_*\| + 2\al\|\eta\| \leqslant \left(1 - \frac{\al}{2}\right)\epsilon \|x_*\| + \al\cdot\frac{\epsilon}{2}\|x_*\| = \epsilon\|x_*\|
\end{align*} so $x_t \in \Bcal(x_*,\epsilon\|x_*\|)$. Thus, we can invoke Lemma \ref{convexity_Lemma} and equation \eqref{iterate_bound} for each $t \geqslant 1$. Letting $\tau := 1-\frac{\al}{2}$, starting at $t = 1$ and repeatedly applying \eqref{iterate_bound}, we attain \begin{align*}
    \|x_{t} - x_*\|  \leqslant \tau^t\|x_{0}-x_*\| + 2\al(\tau^t + \tau^{t-1} + \dots + 1)\|\eta\| \leqslant \tau^t\|x_{0}-x_*\| + \frac{2\al}{1-\tau}\|\eta\|.
\end{align*} Plugging in the definition of $\tau$ yields the desired inequality.
\end{proof}

Finally, using recent tools developed in \cite{Daskalaskisetal2020}, we show that Gaussian matrices satisfy the MDC with high probability with $m = \Omega(n)$ sample complexity.

\begin{proposition} \label{MDC_prop} Fix $0 < \epsilon < 1$.
Suppose $A \in \R^{m \times n}$ has i.i.d. $\mathcal{N}(0,1/m)$ entries. If $m \geqslant C_{\epsilon} n$, then $A$ satisfies the MDC with constant $\epsilon$ with probability at least $1 - \exp(-cm\epsilon^2/2)- m\exp(-n/8) - \exp(-m/2)$. Here $C_{\epsilon} = \Omega(\epsilon^{-2}\log(\epsilon^{-1}))$ and $c$ is a universal constant.
\end{proposition}

\noindent Combining this result with Theorem \ref{conv_to_min} with $ \epsilon = 0.001$ proves Theorem \ref{main_prob_conv_result}.

Hence this shows that the MDC is sufficient for local convergence of subgradient descent with proper initialization. Moreover, the conclusion holds for generic measurements with high probability as soon as $m = \Omega(n)$. We emphasize that the MDC is a global property concerning the measurement matrix. Hence Proposition \ref{MDC_prop} implies one has uniform concentration of subgradients to their expectation with optimal sample complexity. Extending this local convergence result to a result about convergence of subgradient descent with generic initialization is an interesting future direction, as shown in recent works \cite{Chenetal19, TanVershynin19}. 

\section{Proofs} \label{proofs_section}

In this section, we prove Lemma \ref{convexity_Lemma} and Proposition \ref{MDC_prop}. We first introduce some notation used in the proofs. Let $[n] := \{1,\dots,n\}$. Let $\mathcal{B}(y,r) :=\{x \in \R^n : \|x-y\|\leqslant r\}$ and $\Bcal:=\{x \in \R^n : \|x\|\leqslant 1\}$. For $x \in \R^n \setminus\{0\}$, let $\hat{x} := x/\|x\|$. Let $\one_{\{E\}}$ be the indicator function on the event $E$. For a random variable $X$, let $X|(E)$ be the random variable $X$ conditioned on the event $E$. Let $I_n$ be the $n \times n$ identity matrix. Let $\mathcal{S}^{n-1}$ denote the unit sphere in $\R^n$. We write $\gamma = \Omega(\delta)$ when $\gamma \geqslant C\delta$ for some positive constant $C$. Similarly, we write $\gamma = O(\delta)$ when $\gamma \leqslant C \delta$ for some positive constant $C$. 

For a locally Lipschitz function $f : \mathcal{X} \rightarrow \R$ from a Hilbert space $\mathcal{X}$ to $\R$, the Clarke generalized directional derivative  \cite{Clason2017} of $f$ at $x \in \mathcal{X}$ in the direction $u$ is defined by 
    $$f^o(x;u) := \limsup_{y \rightarrow x, t \downarrow 0} \frac{f(y+tu) - f(y)}{t}.$$ Then the generalized subdifferential of $f$ at $x$ is defined as $$\partial f(x) := \{v \in \R^n : \langle v, u \rangle \leqslant f^o(x;u),\ \forall u \in \mathcal{X}\}.$$
 Any $v_{x,x_*} \in \partial f(x)$ is called a subgradient of $f$ at $x$. When $f$ is differentiable at $x$, $\partial f(x) = \{\nabla f(x)\}$. In the proofs, we will make use of the following fact concerning the Clarke subdifferential of the objective function $f$. Since $f$ is piecewise quadratic, Theorem 9.6 from \cite{Clason2017} asserts that for any $x \in \R^n$, $\partial f(x)$ can be written equivalently as \begin{align}
    \partial f(x) = \text{conv}(v_1,v_2,\dots,v_{s})  \label{subdifferential_definition}
\end{align} where $\text{conv}(\cdot)$ denotes the convex hull of $v_1,\dots,v_s$, $s$ is the number of quadratic functions adjoint to $x$, and $v_{\ell}$ is the gradient of the $\ell$-th quadratic function of $f$ at $x$. For each $v_{\ell}$, there exists a $w_{\ell}$ and a sufficiently small $\delta_{\ell} > 0$ such that $f$ is differentiable at $x + \delta_{\ell} w_{\ell}$ and $v_{\ell} = \lim_{\delta_{\ell} \downarrow 0} \nabla f(x + \delta_{\ell}w_{\ell})$.

\subsection{Convexity property of objective} \label{proofs_section_conv}

Here we prove Lemma \ref{convexity_Lemma}, the convexity-like property around the minimizer. In essence, it states that when iterates are near the minimizers, all subgradients point towards the true solution.

\begin{proof}[Proof of Lemma \ref{convexity_Lemma}] We consider the case $x \in \mathcal{B}(x_*,\epsilon\|x_*\|)$ as the case $x \in \mathcal{B}(-x_*,\epsilon\|x_*\|)$ is similar. Suppose $f$ is differentiable at $x$. First, note the MDC implies that $\|A_x^{\T}A_x-I_n\|\leqslant \epsilon$. Moreover, for any $x,z \in \R^n$, $\|A_xz\|^2\leqslant |\langle A_x^{\T}A_xz,z\rangle - \|z\|^2| + \|z\|^2 \leqslant (1+\epsilon)\|z\|^2$. Hence $\|A_x\|\leqslant2$ for all $x \in \R^n$ when $\epsilon < 1$. Thus, we have \begin{align}
    \|v_{x,x_*} - (x-x_*)\| & \leqslant \|A_{x}^{\T}(A_x - A_{x_*})x_*\| +\|A_{x}^{\T}A_x(x - x_*) - (x-x_*)\| + \|A_{x}^{\T}\eta\| \nonumber\\
    & \leqslant 2\|(A_x - A_{x_*})x_*\| + \epsilon\|x-x_*\| + 2\|\eta\|. \label{v_xx*_to_x_x*_firstbound}
\end{align}

We now show that for sufficiently small $\epsilon$, $\|(A_x - A_{x_*})x_*\|\leqslant 1/8\|x-x_*\|$. Letting $\{a_i\}_{i=1}^m$ denote the rows of $A$, observe that \begin{align*}
    \left\|(A_{x} - A_{x_*})x_*\right\|^2
    & = \sum_{i=1}^m\left(\sgn(\langle a_i, x\rangle) - \sgn(\langle a_i , x_*\rangle )\right)^2\langle a_i, x_*\rangle^2 \\
    & \leqslant \sum_{i=1}^m\left(\sgn(\langle a_i, x\rangle) - \sgn(\langle a_i , x_*\rangle)\right)^2\langle a_i, (x -x_*)\rangle^2 \\
    & = \|A_{x}(x - x_*)\|^2 + \|A_{x_*}(x - x_*)\|^2 - 2\langle x - x_*,A_{x}^\top A_{x_*}(x - x_*)\rangle. 
\end{align*} Since $\|A_x^{\T}A_x-I_n\|\leqslant \epsilon$, we have $\|A_x(x-x_*)\|^2 \leqslant (1+\epsilon)\|x-x_*\|^2$. The same upper bound holds for $\|A_{x_*}(x-x_*)\|^2.$ We now bound $2\langle x - x_*,A_{x}^\top A_{x_*}(x - x_*)\rangle$ from below. By the MDC, we have \begin{align}
     |\langle x - x_*,(A_{x}^\top A_{x_*} - \Phi_{x,x_*})(x - x_*)\rangle| \leqslant \epsilon \|x - x_*\|^2. \label{A_x_A_x*_close_to_Phi}
\end{align} Since $x \in \Bcal(x_*,\epsilon\|x_*\|)$, we have that $|\theta_{x,x_*}|\leqslant2\epsilon$. Hence $\Phi_{x,x_*}$ is approximately an isometry since $$\left\|\Phi_{x,x_*} - I_n\right\|  \leqslant \frac{2|\theta_{x,x_*}|}{\pi}\|I_n\| + \frac{2|\sin \theta_{x,x_*}|}{\pi}\|M_{\hat{x} \leftrightarrow \hat{x}_{*}}\| \leqslant \frac{8\epsilon}{\pi}$$ where we used $\|M_{z\leftrightarrow w}\| \leqslant 1$ for all $z,w \in \Scal^{n-1}$. Combining this with \eqref{A_x_A_x*_close_to_Phi}, we have $2\langle x - x_*,A_{x}^\top A_{x_*}(x - x_*)\rangle \geqslant \left(2 - \frac{16\epsilon}{\pi}- 2\epsilon\right)\|x - x_*\|^2.$ Thus we attain
    \begin{align*}
        \|(A_x - A_{x_*})x_*\|^2 &\leqslant  \|A_{x}(x - x_*)\|^2 + \|A_{x_*}(x - x_*)\|^2 - 2\langle x - x_*,A_{x}^\top A_{x_*}(x - x_*)\rangle \\
        & \leqslant \left(2 + 2\epsilon -2 +\frac{16\epsilon}{\pi}+ 2\epsilon \right)\|x-x_*\|^2 \\
        & = \left(4\epsilon + \frac{16\epsilon}{\pi}\right)\|x-x_*\|^2.
    \end{align*} Finally, choosing $\epsilon$ so that $\epsilon \leqslant 0.001$, we conclude  $$2\|(A_x-A_{x_*})x_*\|\leqslant2\sqrt{4\epsilon + 16\epsilon/\pi}\|x-x_*\|\leqslant \frac{1}{4}\|x-x_*\|.$$ Combining this inequality, $\epsilon < 1/4$, and \eqref{v_xx*_to_x_x*_firstbound} shows $\|v_{x,x_*} - (x-x_*)\|\leqslant 1/2\|x-x_*\| + 2\|\eta\|$. 
    
    Finally, for non-differentiable $x$, recall that by \eqref{subdifferential_definition} we can write $v_{x,x_*} = \sum_{\ell=1}^s c_{\ell}v_{\ell}$ where $c_{\ell} \geqslant 0$, $\sum_{\ell=1}^s  c_{\ell} = 1$, and $v_{\ell} =\lim_{\delta_{\ell} \downarrow 0} \nabla f(x+\delta_{\ell} w_{\ell})$ for some $w_{\ell}\in \R^n$. Then, using $\sum_{\ell=1}^s c_{\ell}=1$ and our result for differentiable points, we conclude that for $x \in \mathcal{B}(x_*,\epsilon\|x_*\|)$,
    \begin{align*}
        \|v_{x,x_*} - (x-x_*)\| \leqslant \sum_{\ell=1}^s c_{\ell}\|v_{\ell}-(x-x_*)\| 
        & \leqslant \sum_{\ell=1}^s c_{\ell} \lim_{\delta_{\ell}\downarrow 0}\|\nabla f(x+\delta_{\ell} w_{\ell}) - (x+\delta_{\ell}w_{\ell}-x_*)\| \\
        & \leqslant \frac{1}{2}\|x-x_*\| + 2\|\eta\|.
    \end{align*} \end{proof}

\subsection{Gaussian Matrices Satisfy the MDC} \label{MDC_section}

To show that $A$ satisfies the MDC, we will use novel probabilistic tools developed in \cite{Daskalaskisetal2020}, which improved the sample complexity required for Gaussian matrices to satisfy a related concentration result introduced in \cite{HV2017} known as the Weight Distribution Condition. We first write $A_x^{\T}A_y$ in a more convenient form. For $v \in \R^n$, let $\diag(v > 0)$ denote the diagonal matrix whose $i$-th entry is $1$ if $v_i > 0$ and $0$ otherwise. Define $\diag(v < 0)$ analogously. For $x \in \R^n$, let $A_{+,x}:=\diag(Ax > 0)A$ and $A_{-,x} :=\diag(Ax < 0)A$. Since $\sgn(b) = \one_{\{b > 0\}} - \one_{\{b < 0\}}$ for any $b \in \R$, observe that  $$A_x^{\T}A_y = A_{+,x}^{\T}A_{+,y} + A_{-,x}^{\T}A_{-,y} -A_{+,x}^{\T}A_{-,y}- A_{-,x}^{\T}A_{+,y}.$$ We will establish concentration of each term separately. For the first term, \cite{Daskalaskisetal2020} recently showed that concentration is possible when $m = \Omega(n)$: \begin{lemma}[Theorem 3.2 in \cite{Daskalaskisetal2020}] \label{plus_plus_conc_result} Fix $\epsilon > 0$. If $A \in \R^{m \times n}$ has i.i.d. $\N(0,1/m)$ entries and $m \geqslant C\epsilon^{-2}\log(\epsilon^{-1})n$, then with probability at least $1 - \exp(-cm\epsilon^2/2) - m\exp(-n/8) - \exp(-m/2)$, we have $$\|A_{+,x}^{\T}A_{+,y} - Q_{x,y}\|\leqslant \epsilon\ \forall\ x,y \in \R^n$$ where $Q_{x,y} := 
    \frac{\pi - \theta_{x,y}}{2\pi}I_n + \frac{\sin\theta_{x,y}}{2\pi}M_{\hat{x}\leftrightarrow\hat{y}}$ if $x,y\neq 0$ and $0_{n \times n}$ otherwise. Here $C$ and $c$ are absolute constants.

\end{lemma}

\noindent An elementary calculation shows $\E[A_{+,x}^{\T}A_{+,y}] = Q_{x,y}$. Also by symmetry, $\E[A_{-,x}^{\T}A_{-,y}] = Q_{x,y}$. By applying a nearly identical argument as in \cite{Daskalaskisetal2020}, the analogous result for $A_{-,x}^{\T}A_{-,y}$ holds. \begin{lemma} \label{neg_neg_conc_result} Fix $\epsilon > 0$. If $A \in \R^{m \times n}$ has i.i.d. $\N(0,1/m)$ entries and $m \geqslant C\epsilon^{-2}\log(\epsilon^{-1})n$ then with probability at least $1 - \exp(-cm\epsilon^2/2) - m\exp(-n/8) - \exp(-m/2)$, we have $$\|A_{-,x}^{\T}A_{-,y} - Q_{x,y}\|\leqslant \epsilon\ \forall\ x,y \in \R^n.$$ Here $C$ and $c$ are absolute constants.\end{lemma}

We now extend the argument in \cite{Daskalaskisetal2020} for $A_{+,x}^{\T}A_{-,y}$. Note that a result for  $A_{-x}^{\T}A_{+,y}$ would be identical. Observe that $$\E[A_{+,x}^{\T}A_{-,y}] = H_{x,y}:= \frac{\theta_{x,y}}{2\pi}I_n - \frac{\sin \theta_{x,y}}{2\pi}M_{\hat{x}\leftrightarrow \hat{y}}.$$
We will prove the following: \begin{lemma} \label{Daskalaskis_extension}
Fix $\epsilon > 0$. If $A \in \R^{m \times n}$ has i.i.d. $\N(0,1/m)$ entries and $m \geqslant C\epsilon^{-2}\log(\epsilon^{-1})n$ then with probability at least $1 - \exp(-cm\epsilon^2/2) - m\exp(-n/8) - \exp(-m/2)$, we have $$\|A_{+,x}^{\T}A_{-,y} - H_{x,y}\|\leqslant \epsilon\ \forall\ x,y \in \R^n.$$ Here $C$ and $c$ are absolute constants.
\end{lemma} \noindent Note that this would complete Proposition \ref{MDC_prop} by observing $\Phi_{x,y} = 2Q_{x,y} - 2H_{x,y}$ and combining Lemmas \ref{plus_plus_conc_result}, \ref{neg_neg_conc_result}, \ref{Daskalaskis_extension}, and an analogous result for $A_{-x}^{\T}A_{+,y}$, each satisfied with $\epsilon/4$.

The main probabilistic tool in the proof of Lemma \ref{Daskalaskis_extension} is a result concerning concentration of pseudo-Lipschitz functions. Pseudo-Lipschitzness can be considered as a relaxation of standard Lipschitz continuity but with particular attention towards which sets a function is Lipschitz with respect to. When the sets are balls, then the notion of pseudo-Lipschitzness reduces to standard Lipschitzness. Prior to stating the result, we require the following definitions.
\begin{definition}[$(\delta,\gamma)$-wide system]
A set system $\{B_t \subseteq \R^n : t \in \Theta\}$ is $(\delta,\gamma)$-wide if $B_t = -B_t$, $B_t$ is convex, and $\text{Vol}(B_t \cap \delta \mathcal{B}) \geqslant \gamma \text{Vol}(\delta \mathcal{B})\ \forall\ t \in \Theta.$
\end{definition}
\begin{definition}[pseudo-Lipschitz function] Suppose there exists a $(\delta,\gamma)$-wide system $\{B_t \subseteq \R^n : t \in \Theta\}$ such that $|g_t(x) - g_t(y) |\leqslant \epsilon$ for any $t \in \Theta$ and $x,y \in (\R^n)^d$ with $x_i - y_i \in B_t$ for all $i \in [d]$. Then we say that $\{g_t\}_{t \in \Theta}$ is $(\epsilon,\delta,\gamma)$-pseudo-Lipschitz.

\end{definition}

\noindent Note here that a function is pseudo-Lipschitz with respect to a \textit{particular} system of sets. The following theorem establishes favorable concentration for pseudo-Lipschitz functions.

\begin{theorem}[Theorem 4.4 in \cite{Daskalaskisetal2020}] \label{pseudo_lipschitz_theorem}
Let $\theta$ be a random variable taking values in $\Theta$. Let $\{g_t : (\R^n)^{d} \rightarrow \R : t \in \Theta\}$ be a function family and let $h : (\R^n)^{d} \rightarrow \R$ be a function. Let $\epsilon,\gamma,D > 0$ and $\delta \in (0,1)$. Define the spherical shell $\mathcal{H} := (1 + \delta/2)\Bcal \setminus (1 - \delta/2)\Bcal$ in $\R^n$. Suppose: \begin{enumerate}
    \item For any fixed $x \in \mathcal{H}^d$, $\Pro_{\theta}(g_{\theta}(x) \leqslant h(x) + \epsilon) \geqslant 1 - p,$
    \item $\{g_t\}_{t \in \Theta}$ is $(\epsilon,\delta,\gamma)$-pseudo-Lipschitz,
    \item $|h(x)-h(y)|\leqslant D$ whenever $x \in (\Scal^{n-1})^d$, $y \in (\R^n)^d$, and $\|y_i - x_i\|\leqslant \delta$ for all $i \in [d]$. 
\end{enumerate} Then \begin{align*}
    \Pro_{\theta}\left(g_{\theta}(x) \leqslant h(x) + 2\epsilon + D,\ \forall\ x \in (\Scal^{n-1})^d\right) \geqslant 1 - \gamma^{-2d}(4/\delta)^{2dn}p.
\end{align*}
\end{theorem}

\subsubsection{Proof of Lemma \ref{Daskalaskis_extension}}

For ease of exposition, assume the entries of $A$ are i.i.d. $\N(0,1)$. The main idea is that we will concentrate Lipschitz approximations of $A_{+,x}^{\T}A_{-,y}$ that are upper and lower bounds with respect to the semidefinite ordering. For $\epsilon \in (0,1)$, define the following continuous relaxations of $\one_{\{t > 0\}}$: \begin{align*}
    \varphi^{+}_{-\epsilon}(t) := \begin{cases}
    0 & t \leqslant - \epsilon \\
    1 + t/\epsilon & - \epsilon < t \leqslant 0 \\
    1 & t > 0
    \end{cases}\ \text{and}\ \varphi^{+}_{\epsilon}(t) := \begin{cases}
    0 & t < 0 \\
     t/\epsilon & 0 \leqslant t < \epsilon \\
    1 & t \geqslant \epsilon
    \end{cases}.
\end{align*} Analogously define the following continuous relaxations for $\one_{\{t < 0\}}$: \begin{align*}
    \varphi^{-}_{-\epsilon}(t) := \begin{cases}
    1 & t \leqslant - \epsilon \\
    - t/\epsilon & - \epsilon < t \leqslant 0 \\
    0 & t > 0
    \end{cases}\ \text{and}\ \varphi^{-}_{\epsilon}(t) := \begin{cases}
    1 & t < 0 \\
    1- t/\epsilon & 0 \leqslant t < \epsilon \\
    0 & t \geqslant \epsilon
    \end{cases}.
\end{align*} Then we have that for all $t \in \R$, $
    \varphi^{+}_{\epsilon}(t) \leqslant \one_{\{t > 0\}} \leqslant \varphi^{+}_{-\epsilon}(t)$ and $\varphi^{-}_{-\epsilon}(t) \leqslant \one_{\{t < 0\}} \leqslant \varphi^{-}_{\epsilon}(t)$.
For $V \in \R^{m \times n}$ with rows $v_i$ for $i \in [m]$ and $x,y\in\R^n$, define
    $$G_{V,\text{up}}(x,y)  := \sum_{i=1}^m \varphi_{-\epsilon}^{+}(\langle v_i,x\rangle)\varphi_{\epsilon}^{-}(\langle v_i,y\rangle)v_iv_i^{\T}$$ and $$G_{V,\text{low}}(x,y) := \sum_{i=1}^m \varphi_{\epsilon}^{+}(\langle v_i,x\rangle)\varphi_{-\epsilon}^{-}(\langle v_i,y\rangle)v_iv_i^{\T}.$$
 Note that for any $x,y \in \R^n$, $G_{A,\text{low}}(x,y) \preceq  A_{+,x}^{\T}A_{-,y} \preceq G_{A,\text{up}}(x,y)$ so it suffices to upper bound $G_{A,\text{up}}(x,y)$ and lower bound $G_{A,\text{low}}(x,y)$ uniformly. For the upper bound, we will prove the following:
\begin{proposition}
Fix $0 < \epsilon < 1$. Suppose $A \in \R^{m \times n}$ has i.i.d. $\N(0,1)$ entries. Then if $m \geqslant C\epsilon^{-2}\log(\epsilon^{-1}) n$, we have that with probability at least $1-\exp(-cm\epsilon^2/2) - m\exp(-n/8)-\exp(m/2)$, $$G_{A,\text{up}}(x,y) \preceq mH_{x,y} + m\epsilon  I_n\ \forall\ x,y \neq 0.$$ Here $C$ and $c$ are absolute constants.
\end{proposition} 

The central argument can be broken down into three steps and directly follows \cite{Daskalaskisetal2020}. We first show that the function $g_V(x,y) := \frac{1}{m}\langle u, G_{V,\text{up}}(x,y)u\rangle$ is $(\epsilon,\delta,\gamma)$-pseudo-Lipschitz for fixed $u \in \Scal^{n-1}$ for appropriate parameters $\epsilon,\delta,$ and $\gamma$. Second, we use Theorem \ref{pseudo_lipschitz_theorem} to establish, for fixed $u$, concentration of $g_A(x,y)$ uniformly in $x,y$ to $h(x,y) := \langle u, H_{x,y}u\rangle$. Finally, we use a standard $\epsilon$-net argument to establish uniform concentration over $u$, guaranteeing an upper bound on $G_{A,\text{up}}(x,y)$. Throughout the proof, we will operate on the set of matrices $$\Theta := \left\{V \in \R^{m \times n} : \|V\| \leqslant 3\sqrt{m},\ \max_{i \in [m]}\|v_i\|\leqslant \sqrt{2n}\right\}.$$ When $A$ is Gaussian, standard results \cite{Vershynin_notes} show that $A \in \Theta$ with high probability. 
\begin{lemma} \cite{Vershynin_notes} \label{A_in_theta}
Suppose $A \in \R^{m \times n}$ has i.i.d. $\N(0,1)$ entries. Then with probability at least $1 - \exp(-m/2) - m\exp(-n/8)$, we have $\|A\| \leqslant 3\sqrt{m}$ and $\max_{i \in [m]}\|a_i\|\leqslant \sqrt{2n}$.
\end{lemma}

\paragraph{\textbf{Step 1: Establishing pseudo-Lipschitzness.}} We first establish that $\{g_V\}_{V \in \Theta}$ is pseudo-Lipschitz with respect to a particular set system. 
\begin{lemma}\label{pseudo_lipschitzness_of_gM} Fix $\epsilon > 0$ and $u \in \Scal^{n-1}$. For $V \in \Theta$, define $g_V(x,y) := \frac{1}{m}\langle u, G_{V,\text{up}}(x,y)u\rangle$. Then $\{g_V\}_{V \in \Theta}$ is $(2\epsilon,\epsilon^2/82,1/2)$-pseudo-Lipschitz with respect to the set system $\{B_{M,\epsilon^2,u}\}_{V \in \Theta}$ where $$B_{M,\epsilon^2,u} := \left\{z \in \R^n : \sum_{i=1}^m |\langle v_i,z\rangle|\langle v_i,u\rangle^2 \leqslant \epsilon^2m\right\}.$$

\end{lemma}
\begin{proof}
    We first note that it was shown in Lemma 5.5 of \cite{Daskalaskisetal2020} that the set system $\{B_{M,\epsilon^2,u}\}_{V \in \Theta}$ is  $(\epsilon^2/82,1/2)$-wide. We now show that $\{g_V\}_{V \in \Theta}$ is $(2\epsilon,\epsilon^2/82,1/2)$-pseudo-Lipschitz. For $x,y,\tilde{x},\tilde{y} \in \R^n$, suppose $y -\tilde{y} \in B_{M,\epsilon^2,u}$ and $x - \tilde{x} \in B_{M,\epsilon^2,u}$. Then observe that \begin{align*}
        |g_V(x,y) - g_V(\tilde{x},\tilde{y})|
        & \leqslant \frac{1}{m}\sum_{i=1}^m [|\varphi^+_{-\epsilon}(\langle v_i,x\rangle)\varphi^-_{\epsilon}(\langle v_i,y\rangle) - \varphi^+_{-\epsilon}(\langle v_i,\tilde{x}\rangle)\varphi^-_{\epsilon}(\langle v_i,y\rangle)| \\
        & + |\varphi^+_{-\epsilon}(\langle v_i,\tilde{x}\rangle)\varphi^-_{\epsilon}(\langle v_i,y\rangle) - \varphi^+_{-\epsilon}(\langle v_i,\tilde{x}\rangle)\varphi^-_{\epsilon}(\langle v_i,\tilde{y}\rangle)|]\langle v_i,u\rangle^2 \\
        & \leqslant \frac{1}{m}\sum_{i=1}^m [|\varphi^+_{-\epsilon}(\langle v_i,x\rangle) - \varphi^+_{-\epsilon}(\langle v_i,\tilde{x}\rangle)| \\
        & + |\varphi^-_{\epsilon}(\langle v_i,y\rangle) - \varphi^-_{\epsilon}(\langle v_i,\tilde{y}\rangle)|]\langle v_i,u\rangle^2 \\
        & \leqslant \frac{1}{m \epsilon}\sum_{i=1}^m[|\langle v_i, x- \tilde{x}\rangle| + |\langle v_i,y-\tilde{y}\rangle|]\langle v_i,u\rangle^2 \\
        & \leqslant 2\epsilon.
    \end{align*} In the first inequality, we used the triangle inequality. In the second, we used $|\varphi^+_{-\epsilon}(t)|,|\varphi^-_{\epsilon}(t)|\leqslant 1$. In the third, we used the fact that $\varphi^+_{-\epsilon}$ and $\varphi^-_{\epsilon}$ are both $1/\epsilon$-Lipschitz. In the last inequality, we used the assumptions $y -\tilde{y} \in B_{M,\epsilon^2,u}$ and $x - \tilde{x} \in B_{M,\epsilon^2,u}$. 
\end{proof}

\paragraph{\textbf{Step 2: Point-wise Concentration.}} We now show that, for fixed $u \in \Scal^{n-1}$, $g_V(x,y)$ concentrates around $h(x,y)$ uniformly in $x,y$ by an application of Theorem \ref{pseudo_lipschitz_theorem}. 

\begin{lemma}\label{pointwise_conc_gM} Fix $\epsilon > 0$ and $u \in \Scal^{n-1}$. Let $A \in \R^{m \times n}$ have i.i.d. $\N(0,1)$ entries. Define $\theta := A|(A \in \Theta)$. There exist absolute constants $c, K,$ and $\tilde{C}$ such that $$\Pro_{\theta}\left(g_{\theta}(x,y) \leqslant h(x,y) + K\epsilon \ \forall\ x,y \neq 0\right) \geqslant 1 - (\tilde{C}/\epsilon)^{8n}\exp(-cm\epsilon^2).$$
\end{lemma}
\begin{proof} We will first bound $\E[G_{A,\text{up}}(x,y)]$. Observe that for any $t \in \R$, we have $\varphi^+_{-\epsilon}(t) \leqslant \one_{\{t \geqslant - \epsilon\}}$ and $\varphi^-_{\epsilon}(t) \leqslant \one_{\{t \leqslant \epsilon\}}.$ This implies that for any $t_1,t_2 \in \R$, $$\varphi^+_{-\epsilon}(t_1)\varphi^-_{\epsilon}(t_2) \leqslant \one_{\{t_1 \geqslant-\epsilon\}}\one_{\{t_2 \leqslant\epsilon\}}  \leqslant \one_{\{t_1 > 0\}}\one_{\{t_2 < 0\}} + \one_{\{-\epsilon \leqslant t_1 \leqslant 0\}} + \one_{\{0 \leqslant t_2 \leqslant \epsilon\}}.$$ Thus \begin{align*}
    \E[G_{A,\text{up}}(x,y)] 
    & \preceq \E\left[\sum_{i=1}^m (\one_{\{\langle a_i,x\rangle > 0\}}\one_{\{\langle a_i,y\rangle < 0\}} + \one_{\{-\epsilon \leqslant\langle a_i,x\rangle \leqslant 0\}} + \one_{\{0 \leqslant\langle a_i,y\rangle \leqslant \epsilon\}})a_ia_i^{\T}\right] \\
    & = mH_{x,y} + m \E[\one_{\{-\epsilon \leqslant\langle a,x\rangle \leqslant 0\}}aa^{\T}] + m \E[\one_{\{0 \leqslant\langle a,y\rangle \leqslant \epsilon\}}aa^{\T}]
\end{align*} where $a \sim \N(0,I_n)$. It was shown in Lemma 12 of \cite{HV2017} that $\E[\one_{\{-\epsilon \leqslant\langle a,x\rangle \leqslant 0\}}aa^{\T}] \preceq \frac{\epsilon}{2\|x\|}I_n\ \forall\ x \neq 0.$ An analogous bound shows $\E[\one_{\{0 \leqslant\langle a,y\rangle \leqslant \epsilon\}}aa^{\T}] \preceq \frac{\epsilon}{2\|y\|}I_n$ for all $y \neq 0$. Hence we attain \begin{align*}
     \E[G_{A,\text{up}}(x,y)] & \preceq mH_{x,y} + m \left(\frac{\epsilon}{2\|x\|} + \frac{\epsilon}{2\|y\|}\right)I_n\ \forall\ x,y \neq 0.
\end{align*} This implies $\E[g_A(x,y)] \leqslant h(x,y) + 2\epsilon$ for fixed $x,y \in \R^n$ with $\|x\|,\|y\|\geqslant 1/2$.

Now, we show the probability bound. First consider fixed $x,y \in \R^n$ with $\|x\|,\|y\|\geqslant 1/2.$  Observe that 
    $g_A(x,y) = \frac{1}{m}\sum_{i=1}^m\varphi_{-\epsilon}^{+}(\langle a_i,x\rangle)\varphi_{\epsilon}^{-}(\langle a_i,y\rangle)\langle a_i,u\rangle^2$ is a sum of sub-exponential random variables. Hence by Bernstein's inequality, we have for some absolute constant $c$ and any $\beta > 0$, $\Pro\left(g_A(x,y) - \E[g_A(x,y)] > \beta\right) \leqslant 2 \exp(-cm \min(\beta,\beta^2)).$ Taking $\beta = \epsilon$ and using $\E[g_A(x,y)] \leqslant h(x,y) + 2\epsilon$, we get $\Pro\left(g_A(x,y) > h(x,y) + 3\epsilon\right) \leqslant 2\exp(-cm\epsilon^2).$ Since $\Pro(A \in \Theta) \geqslant 1/2$, conditioning on the event $A \in \Theta$ at most doubles the failure probability so we attain \begin{align}\Pro\left(g_{\theta}(x,y) \leqslant h(x,y) + 3\epsilon\right) \geqslant 1 - 4\exp(-cm\epsilon^2). \label{fixed_xy_prob_bound}\end{align}

To establish uniform concentration in $x,y$, we note that by a simple modification to Lemma 27 in \cite{HLV18}, we have that $H_{x,y}$ is $L$-Lipschitz with respect to $x,y \in \Scal^{n-1}$ where $L = 22/\pi$. Hence $|h(x,y) - h(\tilde{x},\tilde{y})|\leqslant L\epsilon$ if $\|x-\tilde{x}\| \leqslant \epsilon$ and $\|y-\tilde{y}\|\leqslant \epsilon$.  Thus the result follows by applying Theorem \ref{pseudo_lipschitz_theorem} to $\{g_V\}_{V \in \Theta}$ and $\theta := A|(A \in \Theta)$. By Lemma \ref{pseudo_lipschitzness_of_gM}, $\{g_V\}_{V \in \Theta}$ is $(2\epsilon,\epsilon^2/82,1/2)$-pseudo-Lipschitz. We can then take $p = \exp(-cm\epsilon^2)$ by \eqref{fixed_xy_prob_bound} and $D = 2L\epsilon$. \end{proof}

\paragraph{\textbf{Step 3: Uniform Concentration.}} The last step is to get a uniform bound over all $u \in \Scal^{n-1}$. Augmenting our notation, let $g_V(x,y,u) := \frac{1}{m}\langle u,G_{V,\text{up}}(x,y)u\rangle$ and $h(x,y,u) := \langle u, H_{x,y}u\rangle.$
\begin{lemma} \label{Daskalaskis_ext_upper_bound}
Fix $\epsilon > 0$. Suppose $A \in \R^{m \times n}$ has i.i.d. $\N(0,1)$ entries. There exist absolute constants $c$ and $C$ such that if $m \geqslant C\epsilon^{-2}\log(\epsilon^{-1})n$, then with probability $1 - \exp(-cm\epsilon^2/2) - m\exp(-n/8) - \exp(-m/2)$, $$G_{A,\text{up}}(x,y) \preceq mH_{x,y} + m\epsilon\ \forall x,y\neq0.$$
\end{lemma}
\begin{proof} We first show that $g_V(x,y,u)$ is $18$-Lipschitz with respect to $u \in \Scal^{n-1}$ when $V \in \Theta$. Fix $x,y \neq 0$. Observe that for $u,w \in \Scal^{n-1}$, \begin{align}
    |g_V(x,y,u) - g_V(x,y,w)| \leqslant \frac{1}{m}\sum_{i=1}^m |\langle v_i,u\rangle^2-\langle v_i,w\rangle^2| 
    & \leqslant\frac{1}{m}\sum_{i=1}^m|\langle v_i,u-w\rangle||\langle v_i,u+w\rangle|\nonumber \\
    & \leqslant \frac{1}{m}\|V(u-w)\|\|V(u+w)\|\nonumber\\
    &\leqslant 18\|u-w\| \label{g_V_lip_wrt_u}
\end{align} where we used $\|V\|\leqslant 3\sqrt{m}$ along with $u,w\in\Scal^{n-1}$ in the last inequality.

    Let $\N_{\epsilon}\subset \Scal^{n-1}$ be an $\epsilon$-net of cardinality $|\N_{\epsilon}| \leqslant (3/\epsilon)^n$. By Lemma \ref{pointwise_conc_gM} and a union bound, it holds with probability at least $1 - (3/\epsilon)^n(\tilde{C}/\epsilon)^{8n}\exp(-cm\epsilon^2)$ over $\theta = A|(A\in \Theta)$ that for all $x,y \in \Scal^{n-1}$ and $u \in \N_{\epsilon}$, $g_{\theta}(x,y,u) \leqslant h(x,y,u) + K\epsilon.$
For any $x,y,u \in \Scal^{n-1}$, there exists a $w \in \N_{\epsilon}$ with $\|u-w\|\leqslant \epsilon$ so using \eqref{g_V_lip_wrt_u} we get $$g_{\theta}(x,y,u) \leqslant g_{\theta}(x,y,w) + 18\|u-w\| \leqslant h(x,y,w) + 18\epsilon.$$ Since $\|H_{x,y}\|\leqslant 2$, we have that for $u,w \in \Scal^{n-1}$, $|h(x,y,u) - h(x,y,w)| \leqslant 4\|u-w\|$. This further implies $g_{\theta}(x,y,u) \leqslant h(x,y,u) +22\epsilon$. We conclude with probability at least $1 - (3/\epsilon)^n(C/\epsilon)^{8n}\exp(-cm\epsilon^2)$ over $\theta$, the desired inequality holds. Using $\Pro(A \in \Theta) \geqslant 1 - m\exp(-n/8) - \exp(-m/2)$ and taking $m \geqslant C\epsilon^{-2}\log(\epsilon^{-1})n$ achieves the final result with the desired probability.
\end{proof} 

This completes the upper bound on $G_{A,\text{up}}$. The lower bound on $G_{A,\text{low}}$ is identical: 
\begin{lemma} \label{Daskalaskis_ext_lower_bound}
Fix $\epsilon > 0$. Suppose $A \in \R^{m \times n}$ has i.i.d. $\N(0,1)$ entries. There exist absolute constants $c$ and $C$ such that if $m \geqslant C\epsilon^{-2}\log(\epsilon^{-1})n$, then with probability $1 - \exp(-cm\epsilon^2/2) - m\exp(-n/8) - \exp(-m/2)$, $$G_{A,\text{low}}(x,y) \succeq mH_{x,y} - m\epsilon\ \forall x,y\neq0.$$
\end{lemma} 
\noindent Lemma \ref{Daskalaskis_extension} follows by combining Lemma \ref{Daskalaskis_ext_upper_bound} and Lemma \ref{Daskalaskis_ext_lower_bound}.

 \section*{Acknowledgements}
  PH is supported by NSF Grant DMS-2022205 and NSF CAREER Grant DMS-1848087. OL acknowledges support by the NSF Graduate Research Fellowship under Grant No. DGE-1450681.





\vskip2mm


\bibliographystyle{plain}

\bibliography{dense_pr.bib}






          \end{document}